\documentclass[11pt]{article}

\usepackage[utf8]{inputenc} 
\usepackage[T1]{fontenc}    
\usepackage{hyperref}       
\usepackage{url}            
\usepackage{booktabs}       
\usepackage{amsfonts}       
\usepackage{nicefrac}       
\usepackage{microtype}      
\usepackage{color}

\usepackage{amsmath,amsfonts,bm}
\usepackage{mathtools}

\def\1{\bm{1}}

\def\eps{{\epsilon}}

\def\rvB{{\mathbf{B}}}

\def\rmB{{\mathbf{B}}}

\def\vzero{{\bm{0}}}

\def\vgamma{{\bm{\gamma}}}

\def\veps{{\bm{\varepsilon}}}
\def\vbeta{{\bm{\beta}}}

\def\vf{{\bm{f}}}

\def\vh{{\bm{h}}}

\def\vv{{\bm{v}}}
\def\vw{{\bm{w}}}
\def\vx{{\bm{x}}}
\def\vy{{\bm{y}}}
\def\vz{{\bm{z}}}

\def\mzero{{\bm{0}}}

\def\mG{{\bm{G}}}

\def\mI{{\bm{I}}}

\def\mSigma{{\bm{\Sigma}}}

\DeclareMathAlphabet{\mathsfit}{\encodingdefault}{\sfdefault}{m}{sl}
\SetMathAlphabet{\mathsfit}{bold}{\encodingdefault}{\sfdefault}{bx}{n}

\def\gF{{\mathcal{F}}}

\def\sR{{\mathbb{R}}}

\newcommand{\E}{\mathbb{E}}

\usepackage{amsthm}
\usepackage{threeparttable}
\usepackage{amsmath, amssymb}
\usepackage{graphicx}
\usepackage{wrapfig}
\usepackage{enumitem}
\usepackage{multirow, multicol}
\usepackage{threeparttable}
\newtheorem{definition}{Definition}[section]
\newtheorem{theorem}{Theorem}[section]
\newtheorem{corollary}{Corollary}[theorem]
\newtheorem{lemma}[theorem]{Lemma}
\newtheorem{assumption}{Assumption}
\usepackage{url}
\usepackage{commath}
\usepackage{geometry}
\title{Neural SDE: Stabilizing Neural ODE Networks with Stochastic Noise}

\author{Xuanqing Liu \\
  Department of Computer Science, UCLA\\
  \texttt{xqliu@cs.ucla.edu}
   \and
   Tesi Xiao \\
   Department of Statistics, UC Davis \\
   \texttt{texiao@ucdavis.edu}
   \and
   Si Si \\
   Google Research \\
  \texttt{sisidaisy@google.com}
   \and
   Qin Cao \\
   Google Research \\
   \texttt{qincao@google.com}
   \and
   Sanjiv Kumar \\
   Google Research \\
   \texttt{sanjivk@google.com}
   \and
   Cho-Jui Hsieh \\
    Department of Computer Science, UCLA \\
   \texttt{chohsieh@cs.ucla.edu}
}

\begin{document}
\maketitle
\begin{abstract}
Neural Ordinary Differential Equation (Neural ODE) has been proposed as a continuous approximation to the ResNet architecture. Some commonly used regularization mechanisms in discrete neural networks (e.g. dropout, Gaussian noise) are missing in current Neural ODE networks. In this paper, we propose a new continuous neural network framework called Neural Stochastic Differential Equation (Neural SDE) network, which naturally incorporates various commonly used regularization mechanisms based on random noise injection. Our framework can model various types of noise injection frequently used in discrete networks for regularization purpose, such as dropout and additive/multiplicative noise in each block. We provide theoretical analysis explaining the improved robustness of Neural SDE models against input perturbations/adversarial attacks. Furthermore, we demonstrate that the Neural SDE network can achieve better generalization than the Neural ODE and is more resistant to adversarial and non-adversarial input perturbations.
\end{abstract}

\section{Introduction}
Residual neural networks (ResNet)~\cite{he2016deep} are composed of multiple residual blocks
transforming the hidden states according to:
\begin{equation}
    \label{eq:resnet}
    \vh_{n+1}=\vh_n+\vf(\vh_n; \vw_n),
\end{equation}
where $\vh_n$ is the input to the $n$-th layer and $\vf(\vh_n; \vw_n)$ is a non-linear function parameterized by $\vw_n$. Recently, a continuous approximation to the ResNet architecture has been proposed~\cite{chen2018neural}, where the evolution of the hidden state $\vh_t$ can be described as a dynamic system obeying the equation:
\begin{equation}
    \label{eq:neural-ode}
 \vh_t=\vh_s+\int_s^t \vf(\vh_{\tau}, \tau;\vw)\dif\tau,
\end{equation}
where $\vf(\vh_{\tau}, \tau;\vw)$ is the continuous form of the nonlinear function $\vf(\vh_n; \vw_n)$; $\vh_s$ and $\vh_t$ are hidden states at two different time $s\ne t$. A standard ODE solver can be used to solve all the hidden states and final states (output from the neural network), starting from an initial state (input to the neural network).
The continuous neural network described in~\eqref{eq:neural-ode} exhibits several advantages over its discrete counterpart described in~\eqref{eq:resnet}, in terms of memory efficiency, parameter efficiency, explicit control of the numerical error of final output, \textit{etc.}

One missing component in the current Neural ODE network is the various regularization mechanisms commonly employed in discrete neural networks. These regularization techniques have been demonstrated to be crucial in reducing generalization errors, and in improving the robustness of neural networks to adversarial attacks. Many of these regularization techniques are based on stochastic noise injection. For instance, dropout~\cite{srivastava2014dropout} is widely adopted to prevent overfitting; injecting Gaussian random noise during the forward propagation is effective in improving  generalization~\cite{bishop1995training,an1996effects} as well as robustness to adversarial attacks~\cite{liu2018towards,lecuyer2018certified}. However, these regularization methods in discrete neural networks are not directly applicable to Neural ODE network, because Neural ODE network is a deterministic system.

Our work attempts to incorporate the above-mentioned stochastic noise injection based regularization mechanisms to the current Neural ODE network, to improve the generalization ability and the robustness of the network. 
In this paper, we propose a new continuous neural network framework called {\bf Neural Stochastic Differential Equation (Neural SDE)} network, which models stochastic noise injection by stochastic differential equations (SDE). 
In this new framework, we can employ existing techniques from the stability theory of SDE to study the robustness of neural networks. Our results provide theoretical insights to understanding why introducing stochasticity during neural network training and testing leads to improved robustness against adversarial attacks.
Furthermore, we demonstrate that, by incorporating the noise injection regularization mechanism to the continuous neural network, we can reduce overfitting and achieve lower generalization error. For instance, on the CIFAR-10 dataset, we observe that the new Neural SDE can improve the test accuracy of the Neural ODE from 81.63\% to 84.55\%, with other factors unchanged. Our contributions can be summarized as follows: 
\begin{itemize}[noitemsep,leftmargin=*]
    \item We propose a new Stochastic Differential Equation (SDE) framework to incorporate randomness in continuous neural networks. The proposed random noise injection can be used as a drop-in component in any continuous neural networks. Our Neural SDE framework can model various types of noises widely used for regularization purpose in discrete networks, such as dropout (Bernoulli type) and Gaussian noise. 
    \item  Training the new SDE network requires developing different backpropagation approach from the Neural ODE network. We develop a new efficient backpropagation method to calculate the gradient, and to train the Neural SDE network in a scalable way. The proposed method has its roots in stochastic control theory.
    \item We carry out a theoretical analysis of the stability conditions of the Neural SDE network, to prove that the randomness introduced in the Neural SDE network can stabilize the dynamical system, which helps improve the robustness and generalization ability of the neural network. 
    \item We verify by numerical experiments that stochastic noise injection in the SDE network can successfully regularize the continuous neural network models, and the proposed Neural SDE network achieves better robustness and improves generalization performance. 
\end{itemize}
\paragraph{Notations:} Throughout this paper, we use $\vh\in\sR^n$ to denote the hidden states in a neural network, where $\vh_0=\vx$ is the input (also called initial condition) and $y$ is the label. The residual block with parameters $\vw\in\sR^d$ can be written as a nonlinear transform $\vf(\vh_{\tau}, \tau; \vw)$. We assume the integration is always taken from $0$ to $T$. $\rvB_t\in\sR^m$ is $m$-dimensional Brownian motion. $\mG(\vh_{\tau}, \tau;\vv)\in\sR^{n\times m}$ is the diffusion matrix parameterized by $\vv$. Unless stated explicitly, we use $\|\cdot\|$ to represent $\ell_2$-norm for vector and Frobenius norm for matrix.
\section{Related work}
Our work is inspired by the success of the recent Neural ODE network, and we seek to improve the generalization and robustness of Neural ODE, by adding regularization mechanisms crucial to the success of discrete networks.  Regularization mechanisms such as dropout cannot be easily incorporated in the Neural ODE due to its deterministic nature.
\paragraph{Neural ODE} 
The basic idea of Neural ODE is discussed in the previous section, here we briefly review relevant literature. The idea of formulating ResNet as a dynamic system was discussed in~\cite{weinan2017proposal}. A framework was proposed to link  existing deep architectures with discretized numerical ODE solvers~\cite{lu2018beyond}, and was shown to be parameter efficient. These networks adopt layer-wise architecture -- each layer is parameterized by different independent weights. The Neural ODE model~\cite{chen2018neural} computes hidden states in a different way: it directly models the dynamics of hidden states by an ODE solver, with the dynamics parameterized by a shared model. A memory efficient approach to compute the gradients by adjoint methods was developed, making it possible to train large, multi-scale generative networks~\cite{ardizzone2018analyzing,grathwohl2018ffjord}. Our work can be regarded as an extension of this framework, with the purpose of incorporating a variety of noise-injection based regularization mechanisms. Stochastic differential equation in the context of neural network has been studied before, focusing either on understanding how dropout shapes the loss landscape~\cite{sun2018stochastic}, or on using stochastic differential equation as a universal function approximation tool to learn the solution of high dimensional PDEs~\cite{raissi2018forward}. Instead, our work tries to explain why adding random noise boosts the stability of deep neural networks, and demonstrates the improved generalization and robustness. 
\paragraph{Noisy Neural Networks} Adding random noise to different layers is a technique commonly employed in training neural networks. Dropout~\cite{srivastava2014dropout} randomly disables some neurons to avoid overfitting, which can be viewed as multiplying hidden states with Bernoulli random variables. Stochastic depth neural network~\cite{huang2016deep} randomly drops some residual blocks of residual neural network during training time. Another successful regularization for ResNet is Shake-Shake regularization~\cite{gastaldi2017shake}, which sets a binary random variable to randomly switch between two residual blocks during training. More recently, dropblock~\cite{ghiasi2018dropblock} was designed specifically for convolutional layers: unlike dropout, it drops some continuous regions rather than sparse points to hidden states. All of the above regularization techniques are proposed to improve generalization performance. One common characteristic of them is that \emph{they fix the network during testing time}. There is another line of research that focuses on improving robustness to perturbations/adversarial attacks by noise injection. Among them, random self-ensemble~\cite{liu2018towards,lecuyer2018certified} adds Gaussian noise to hidden states during both training and testing time. In training time, it works as a regularizer to prevent overfitting; in testing time, the random noise is also helpful, which will be explained in this paper.
\section{Neural Stochastic Differential Equation}
\begin{wrapfigure}{r}{0.4\textwidth}
    \begin{center}
    \includegraphics[width=0.39\textwidth]{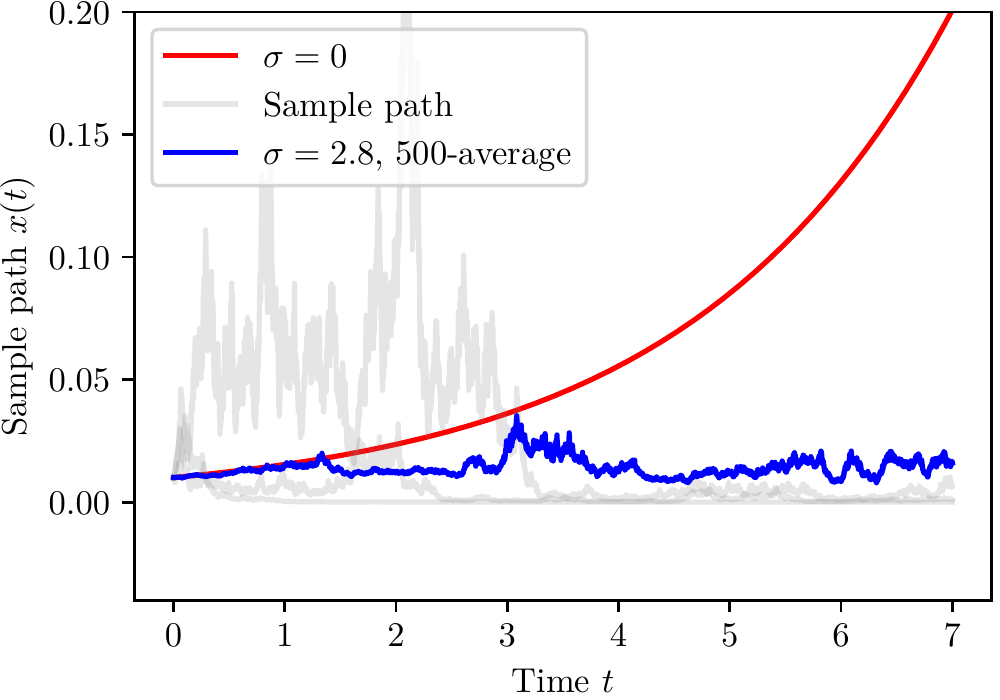}
    \end{center}
    \caption{Toy example. By comparing the simulations under $\sigma=0$ and $\sigma=2.8$, we see adding noise to the system can be an effective way to control $x_t$. Average over multiple runs is used to cancel out the volatility during the early stage.}
    \label{fig:toy_example}
\end{wrapfigure}
In this section, we first introduce our proposed Neural SDE to improve the robustness of Neural ODE. Informally speaking, Neural SDE can be viewed as using randomness as a drop-in augmentation for Neural ODE, and it can include some widely used randomization layers such as dropout and Gaussian noise layer~\cite{liu2018towards}. However, solving Neural SDE is non-trivial, we derive the gradients of loss over model weights. Finally we theoretically analyze the stability conditions of Neural SDE.
\par
Before delving into the multi-dimensional SDE, let's first look at a 1-d toy example to see how SDE can solve the instability issue of ODE. Suppose we have a simple SDE, $\dif x_t=x_t\dif t+\sigma x_t\dif B_t$ with $B_t$ be the standard Brownian motion.
We provide a numerical simulation in Figure~\ref{fig:toy_example} for $x_t$ with different $\sigma$.

When we set $\sigma=0$,  SDE becomes ODE $\dif x_t=x_t\dif t$ and $x_t=c_0e^t$ where $c_0$ is an integration constant. If $c_0\ne 0$ we can see that $x_t\to\pm\infty$. Furthermore, a small perturbation in $x_t$ will be amplified through $t$. This clearly shows instability of ODE. On the other hand, if we instead make $\sigma>1$ (the system is SDE), we have $x_t=c_0\exp\big((1-\sigma^2)t+\sigma B_t\big)\overset{\mathrm{a.s.}}{\to}0$. 
\begin{figure}[htb]
    \centering
    \includegraphics[width=0.8\linewidth]{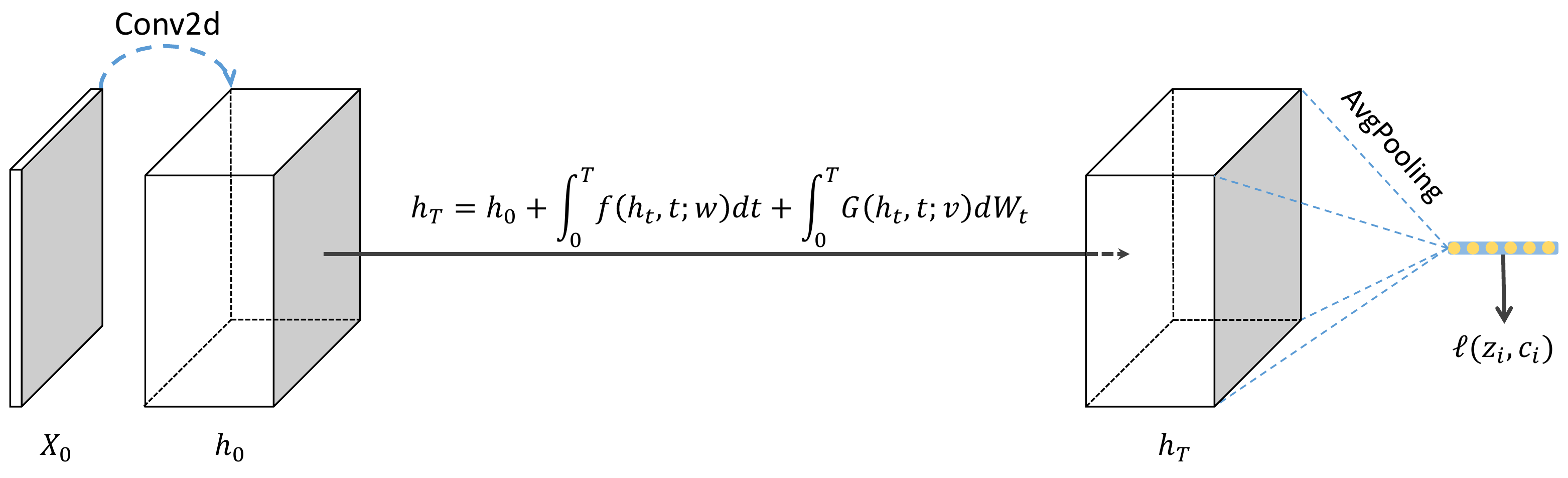}
    \caption{Our model architecture. The initial value of SDE is the output of a convolutional layer, and the value at time $T$ is passed to a linear classifier after average pooling.}
    \label{fig:model_architecture}
\end{figure}
\par
The toy example in Figure~\ref{fig:toy_example} reveals that the behavior of solution paths can change significantly after adding a stochastic term. This example is inspiring because we can control the impact of perturbations on the output by adding a stochastic term to neural networks. 

Figure \ref{fig:model_architecture} shows a sample Neural SDE model architecture, and it is the one used in the experiment. It consists of three parts, the first part is a single convolution block, followed by a Neural SDE network (we will explain the detail of Neural SDE in Section \ref{sec:randomness}) and lastly the linear classifier. We put most of the trainable parameters into the second part (Neural SDE), whereas the first/third parts are mainly for increasing/reducing the dimension as desired. Recall that both Neural ODE and SDE are dimension preserving.
\subsection{Modeling randomness in neural networks}
\label{sec:randomness}
In the Neural ODE system~\eqref{eq:neural-ode}, a slightly perturbed input state will be amplified in deep layers (as shown in Figure~\ref{fig:toy_example}) which makes the system unstable to input perturbation and prone to overfitting. Randomness is an important component in discrete networks (e.g., dropout for regularization) to tackle this issue, however to our knowledge, there is no existing work concerning adding randomness in the continuous neural networks. And it is non-trivial to encode randomness in continuous neural networks, such as Neural ODE, as we need to consider how to add randomness so that to guarantee the robustness, and how to solve the continuous system efficiently. To solve these challenges, motivated by ~\cite{lu2018beyond,sun2018stochastic}, we propose to add a single diffusion term into Neural ODE as:
\begin{equation}
    \label{eq:neural-sde}
    \dif \vh_t=\vf(\vh_t,t;\vw)\dif t +\mG(\vh_t,t;\vv)\dif\rvB_t,
\end{equation}
where $\rvB_t$ is the standard Brownian motion~\cite{oksendal2003stochastic}, which is a continuous time stochastic process such that $\rvB_{t+s}-\rvB_s$ follows Gaussian with mean $0$ and variance $t$; $\mG(\vh_t,t;\vv)$ is a transformation parameterized by $\vv$. 
This formula is quite general, and can include many existing randomness injection models with residual connections under different forms of $\mG(\vh_t,t;\vv)$. As examples, we briefly list some of them below.
\paragraph{Gaussian noise injection:}
Consider a simple example in \eqref{eq:neural-sde} when $\mG(\vh_t, t; \vv)$ is a diagonal matrix, and we can model both additive and multiplicative noise as 
\begin{equation}
    \label{eq:RSE-continuous}
    \begin{aligned}
    \text{additive: }&\dif \vh_t=\vf(\vh_t, t;\vw)\dif t+\mSigma(t)\dif \rvB_t \\
    \text{multiplicative:}&\dif \vh_t=\vf(\vh_t, t;\vw)\dif t+\mSigma(\vh_t, t)\dif \rvB_t, 
    \end{aligned}
\end{equation}
where $\mSigma(t)$ is a diagonal matrix and its diagonal elements control the variance of the noise added to hidden states. This can be viewed as a continuous approximation of noise injection techniques in discrete neural network. For example, the discrete version of the additive noise can be written as
\begin{equation}
    \label{eq:RSE-discrete}
    \vh_{n+1}=\vh_n+\vf(\vh_n;\vw_n)+\mSigma_n\vz_n, \quad\text{with\ \ } \mSigma_n=\sigma_n\mI,\ \vz_n\overset{\text{i.i.d.}}{\sim}\mathcal{N}(0, 1),
\end{equation}
which injects Gaussian noise after each residual block. It has been shown that injecting small Gaussian noise can be viewed as a regularization in neural networks~\cite{bishop1995training,an1996effects}. Furthermore, \cite{liu2018towards,lecuyer2018certified} recently showed that adding a slightly larger noise in one or all residual blocks can improve the adversarial robustness of neural networks. We will provide the stability analysis of \eqref{eq:neural-sde} in Section \ref{sec:robustness}, which provides a theoretical explanation towards the robustness of Neural SDE.

\paragraph{Dropout:} 
Our framework can also model the dropout layer which randomly disables some neurons in the residual blocks. Let us see how to unify dropout under our Neural SDE framework.
First we notice that in the discrete case
\begin{equation}
    \label{eq:dropout-sde}
    \vh_{n+1}=\vh_{n} + \vf(\vh_n;\vw_n)\odot \frac{\vgamma_n}{p}=\vh_{n}+\vf(\vh_n;\vw_n) + \vf(\vh_n;\vw_n)\odot (\frac{\vgamma_n}{p}-\mI),
\end{equation}
where $\vgamma_n\stackrel{\text{i.i.d.}}{\sim}\mathcal{B}(1, p)$ and $\odot$ indicates the Hadamard product. Note that we divide $\gamma_n$ by $p$ in \eqref{eq:dropout-sde} to maintain the same expectation. 
Furthermore, we have 
\begin{equation}\label{eq:split-variance}
\frac{\vgamma_n}{p}-I=\sqrt{\frac{1-p}{p}}\cdot\boxed{\sqrt{\frac{p}{1-p}}\Big(\frac{\gamma_n}{p}-\mI\Big)}\approx \sqrt{\frac{1-p}{p}}\vz_n, \quad \vz_n\stackrel{\text{i.i.d.}}{\sim}\mathcal{N}(0, 1).
\end{equation}
The boxed part above is approximated by standard normal distribution (by matching the first and second order moment). The final SDE with dropout can be obtained by combining \eqref{eq:dropout-sde} with \eqref{eq:split-variance}
\begin{equation}
    \label{eq:dropout-ode-final}
    \dif \vh_{t}=\vf(\vh_t, t;\vw)\dif t+\sqrt{\frac{1-p}{p}}\vf(\vh_t, t;\vw)\odot\dif \rvB_t.
\end{equation}
\paragraph{Others:} 
\cite{lu2018beyond}~includes some other stochastic layers that can be formulated under Neural SDE framework, including shake-shake regularization~\cite{gastaldi2017shake} and stochastic depth~\cite{huang2016deep}. Both of them are used as regularization techniques that work very similar to dropout. 
\subsection{Back-propagating through SDE integral}
To optimize the parameters $\vw$, we need to back-propagate the Neural SDE system. A straightforward solution is to rely on the autograd method derived from chain rule. However, for Neural SDE the chain can be fairly long. If SDE solver discretizes the range $[t_0, t_1]$ to $N$ intervals, then the chain has $\mathcal{O}(N)$ nodes and the memory cost is $\mathcal{O}(N)$.
One challenging part of backpropagation for Neural SDE is to calculate the gradient through SDE solver which could have high memory cost. To solve this issue, we first calculate the \emph{expected loss} conditioning on the initial value $\vh_0$, denoted as $L=\mathbb{E}[\ell(\vh_{t_1})|\vh_0]$. Then our goal is to calculate $\frac{\partial L}{\partial\vw}$. In fact, we have the following theorem (also called \emph{path-wise} gradient~\cite{yang1991monte,gobet2005sensitivity}).
\begin{theorem}
    \label{th:gradient-sde}
    For continuously differentiable loss $\ell(\vh_{t_1})$, we can obtain an unbiased gradient estimator as
    \begin{equation}\label{eq:mc-gradient}
    \widehat{\frac{\partial L}{\partial \vw}}=\frac{\partial\ell(\vh_{t_1})}{\partial\vw}=\frac{\partial \ell(\vh_{t_1})}{\partial\vh_{t_1}}\cdot\frac{\partial\vh_{t_1}}{\partial\vw}.
    \end{equation}
    Moreover, if we define $\vbeta_t=\frac{\partial\vh_{t}}{\partial\vw}$, then $\vbeta_t$ follows another SDE
    \begin{equation}
        \label{eq:adjoint-sde}
        \dif \vbeta_t=\Big(\frac{\partial \vf(\vh_t,t;\vw)}{\partial\vw}+\frac{\partial \vf(\vh_t,t;\vw)}{\partial\vh_t}\vbeta_t\Big)\dif t+\Big(\frac{\partial\mG(\vh_t,t;\vw)}{\partial \vw}+\frac{\partial\mG(\vh_t,t;\vw)}{\partial \vh_t}\vbeta_t\Big)\dif \rvB_t.
    \end{equation}
    It is easy to check that if $\mG\equiv\mzero$, then our Monte-Carlo gradient estimator~\eqref{eq:mc-gradient} falls back to the exact gradient by back-propagation.
\end{theorem}
Similar to the adjoint method in Neural ODE, we will solve \eqref{eq:adjoint-sde} jointly with the original SDE dynamics~\eqref{eq:neural-sde}, this process can be done iteratively without memorizing the middle states, which makes it more memory efficient than autograd ($\mathcal{O}(1)$ \emph{v.s.} $\mathcal{O}(N)$ memory, $N$ is the number of steps determined by SDE solver). 
\subsection{Robustness of Neural SDE}
\label{sec:robustness}
In this section, we theoretically analyze the stability of Neural SDE, showing that the randomness term can indeed improve the robustness of the model against small input perturbation. This also explains why noise injection can improve the  robustness in discrete networks, which has been observed in literature~\cite{liu2018towards,lecuyer2018certified}. First we need to show the existence and uniqueness of solution to \eqref{eq:neural-sde}, we pose following assumptions on drift $\vf$ and diffusion $\mG$.
\begin{assumption}\label{ass:sublinear}
$\vf$ and $\mG$ are at most linear, i.e. $\|\vf(\vx,t)\|+\|\mG(\vx,t)\|\le c_1(1+\|\vx\|)$ for $c_1>0$, $\forall \vx\in\sR^n$ and $t\in \sR^+$.
\end{assumption}

\begin{assumption}\label{ass:lipschitz}
$\vf$ and $\mG$ are $c_2$-Lipschitz: $\|\vf(\vx,t)-\vf(\vy,t)\|+\|\mG(\vx,t)-\mG(\vy,t)\|\le c_2\|\vx-\vy\|$ for $c_2>0$, $\forall \vx,\vy\in\sR^n$ and $t\in\sR^+$.
\end{assumption}
Based on the above assumptions, we can show that the SDE~\eqref{eq:neural-sde} has a unique solution~\cite{oksendal2003stochastic}. We remark that assumption on $\vf$ is quite natural and is also enforced on the original Neural ODE model~\cite{chen2018neural}; as to diffusion matrix $\mG$, we have seen that for dropout, Gaussian noise injection and other random models, both assumptions are automatically satisfied as long as $\vf$ possesses the same regularities.

We analyze the dynamics of perturbation. Our analysis applies not only to the Neural SDE model but also to Neural ODE model, by setting the diffusion term $\mG$ to zero. First of all, we consider initializing our Neural SDE~\eqref{eq:neural-sde} at two slightly different values $\vh_0$ and $\vh_0^{e}=\vh_0+\veps_0$, where $\veps_0$ is the perturbation for $\vh_0$ with  $\|\veps_0\|\le \delta$. So, under the new perturbed initialization $\vh_0^e$, the hidden state at time $t$ follows the same SDE in \eqref{eq:neural-sde}, 
\begin{equation}
    \label{eq:neural-sde-new-init}
    \dif \vh_t^e=\vf(\vh_t^e,t;\vw)\dif t+\mG(\vh_t^e,t;\vv)\dif \rvB'_t, \quad\text{with } \vh_0^e=\vh_0+\veps_0, 
\end{equation}
where $\rvB'_t$ is Brownian motions for the SDE associated with initialization $\vh_0^e$. Then it is natural to analyze how the perturbation $\veps_t=\vh_t^e-\vh_t$ evolves in the long run. Subtracting \eqref{eq:neural-sde} from \eqref{eq:neural-sde-new-init} 
\begin{equation}
    \label{eq:error-SDE}
    \begin{aligned}
    \dif \veps_t&=\big[\vf(\vh_t^e,t;\vw)-\vf(\vh_t,t;\vw)\big]\dif t+\big[\mG(\vh_t^e,t;\vv)-\mG(\vh_t,t;\vv)\big]\dif \rvB_t\\
    &= \vf_{\Delta}(\veps_t,t;\vw)\dif t+\mG_{\Delta}(\veps_t,t;\vv)\dif \rvB_t.
    \end{aligned}
\end{equation}
Here we made an implicit assumption that the Brownian motions $\rvB_t$ and $\rvB'_t$ have the \emph{same} sample path for both initialization $\vh_0$ and $\vh_0^e$, i.e. $\rvB_t=\rvB'_t$ w.p.1. In other words, we focus on the difference of two random processes $\vh_t$ and $\vh_t^e$ driven by the same underlying Brownian motion. So it is valid to subtract the diffusion terms.
\par
An important property of \eqref{eq:error-SDE} is that it admits a trivial solution $\veps_t\equiv \vzero$, $\forall t\in\sR^+$ and $\vw\in\sR^d$. We show that both the drift ($\vf$) and diffusion ($\mG$) are zero under this solution:
\begin{equation}
    \label{eq:zero-drift-diffusion}
    \begin{aligned}
    \vf_{\Delta}(\vzero,t;\vw)&=\vf(\vh_t+\vzero,t;\vw)-\vf(\vh_t,t;\vw)=0,\\
    \mG_{\Delta}(\vzero,t;\vv)&=\mG(\vh_t+\vzero,t;\vw)-\mG(\vh_t,t;\vw)=0.
    \end{aligned}
\end{equation}
The implication of zero solution is clear: for a neural network, if we do not perturb the input data, then the output will never change. However, the solution $\veps_t=\vzero$ can be \emph{highly unstable}, in the sense that for an arbitrarily small perturbation $\veps_0\ne\vzero$ at initialization, the change of output $\veps_T$ can be arbitrarily bad. On the other hand, as shown below, by choosing the diffusion term $\mG$ properly, we can always control $\veps_t$ within a small range. 
\par
In general, we cannot get the closed form solution to a multidimensional SDE but we can still analyze the asymptotic stability through the dynamics $\vf$ and $\mG$. This is essentially an extension of Lyapunov stability theory to a  stochastic system. First we define the notion of stability in the stochastic case. Let $(\Omega, \gF ,P)$ be a complete probability space with filtration $\{\gF_t\}_{t\geq0}$ and  $\rmB_t$ be an $m$-dimensional Brownian motion defined in the probability space, we consider the SDE in~Eq.~\eqref{eq:error-SDE} with initial value $\veps_0$
\begin{equation}
\label{eq:sde}
    \dif \veps_t=\vf_{\Delta}(\veps_t,t)\dif t + \mG_{\Delta}(\veps_t, t)\dif \rvB_t,
\end{equation}
For simplicity we dropped the dependency on parameters $\vw$ and $\vv$. We further assume $\vf_{\Delta}:\sR^n\times\sR_+ \mapsto \sR^n$ and $\mG_{\Delta}: \sR^n\times\sR_+ \mapsto \sR^{n\times m}$ are both Borel measurable. We can show that if assumptions (\ref{ass:sublinear}) and (\ref{ass:lipschitz}) hold for $\vf$ and $\mG$, then they hold for $\vf_{\Delta}$ and $\mG_{\Delta}$ as well (see Lemma \ref{lem:ass-for-fG-Delta} in Appendix), and we know the SDE~\eqref{eq:sde} allows a unique solution $\veps_t$. We have the following Lynapunov stability results from \cite{mao2007stochastic}.
\begin{definition}[Lyapunov stability of SDE]
\label{def:lyapunov}
The solution $\veps_t=\vzero$ of \eqref{eq:sde}:
\begin{itemize}[noitemsep,leftmargin=*]
    \item[A.] is \underline{stochastically stable} if for any $\alpha\in (0,1)$ and $r> 0$, there exists a $\delta=\delta(\alpha, r)>0$ such that $\Pr\{\| \veps_t\| < r \text{ for all } t\geq 0\}\geq 1-\alpha$ whenever $\| \veps_0\|\leq \delta$. Moreover, if for any $\alpha\in (0,1)$, there exists a $\delta=\delta(\alpha)>0$ such that $\Pr\{\lim_{t\rightarrow \infty}\Vert \veps_t \Vert =0 \}\geq 1-\alpha$ whenever $\Vert \veps_0\Vert\leq \delta$, it is said to be \underline{stochastically asymptotically stable}; 
    \item[B.] is \underline{almost surely exponentially stable} if $\underset{t\rightarrow\infty}{\limsup} \frac{1}{t} \log \Vert \veps_t\Vert < 0$ a.s.\footnote{``a.s.'' is the abbreviation for ``almost surely''.} for all $\veps_0\in\sR^n$.
\end{itemize}
\end{definition}

Note that for part A in Definition~\ref{def:lyapunov}, it is hard to quantify how well the stability is and how fast the solution reaches equilibrium. In addition, under assumptions (\ref{ass:sublinear}, \ref{ass:lipschitz}), we have a straightforward result $\Pr\{\veps_t\neq \vzero \text{ for all } t\geq 0\}=1$ whenever $\veps_0 \neq \vzero$ as shown in Appendix (see Lemma \ref{lem:nonzero}). That is, almost all the sample paths starting from a non-zero initialization can never reach zero due to Brownian motion. On the contrary, the almost sure exponentially stability result implies that almost all the sample paths of the solution will be close to zero exponentially fast. We present the following theorem from~\cite{mao2007stochastic} on the almost sure exponentially stability.

\begin{theorem}
    \label{th:exp-stable}
  \textsc{\cite{mao2007stochastic}}  If there exists a non-negative real valued function $V(\veps ,t)$ defined on $\sR^n\times\sR_+$ that has continuous partial derivatives $$V_{1}(\veps,t)\coloneqq \frac{\partial V(\veps,t)}{\partial \veps}, V_{2}(\veps,t)\coloneqq \frac{\partial V(\veps,t)}{\partial t}, V_{1,1}(\veps,t)\coloneqq \frac{\partial^2 V(\veps,t)}{\partial \veps\partial \veps^\top}$$
  and constants $p>0, c_1>0, c_2\in\sR, c_3\geq 0$ such that the following inequalities hold:
    \begin{enumerate}[noitemsep,leftmargin=*]
    \item $c_1\Vert \veps \Vert^p \leq V(\veps, t)$
    \item $\mathcal{L}V(\veps,t)= V_2 (\veps,t)+  V_1(\veps,t) \vf_{\Delta}(\veps,t) + \frac{1}{2}\mathsf{Tr}[\mG_{\Delta}^\top(\veps,t)V_{1,1}(\veps,t)\mG_{\Delta}(\veps,t)]\leq c_2 V(\veps,t)$
    \item $\Vert V_{1}(\veps,t) \mG_{\Delta}(\veps,t) \Vert^2 \geq c_3 V^2(\veps,t)$
    \end{enumerate}
    for all $\veps\neq \vzero$ and $t>0$. Then for all $\veps_0\in \sR^n$,
    \begin{equation}
    \label{eq:exp-stable}
    \underset{t\rightarrow\infty}{\limsup} \frac{1}{t}\log \Vert \veps_t \Vert \leq -\frac{c_3-2c_2}{2p} \quad a.s.
    \end{equation}
 In particular, if $c_3\geq 2c_2$, the solution $\veps_t\equiv \vzero$ is almost surely exponentially stable.
\end{theorem}
We now consider a special case, when the noise is multiplicative $\mG(\vh_t, t) = \sigma\cdot\vh_t $ and $m=1$. The corresponding SDE of perturbation $\veps_t=\vh_t^e-\vh_t$ has the following form
\begin{equation}
    \label{eq:multiplicative-error-SDE}
    \dif \veps_t=\vf_{\Delta}(\veps_t,t;\vw)\dif t+ \sigma\cdot\veps_t  \dif \rvB_t.
\end{equation}
Note that for the deterministic case of \eqref{eq:multiplicative-error-SDE} by setting $\sigma\equiv 0$, the solution may not be stable in certain cases (see Figure \ref{fig:toy_example}). Whereas for general cases when $\sigma>0$, following corollary claims that by setting $\sigma$ properly, we will achieve an (almost surely) exponentially stable system.

\begin{corollary}
    \label{cor:stabilize}
    For \eqref{eq:multiplicative-error-SDE}, if $\vf(\vh_t, t; \vw)$ is $L$-Lipschtiz continuous w.r.t. $\vh_t$, then \eqref{eq:multiplicative-error-SDE} has a unique solution with the property
    $
        \underset{t\rightarrow\infty}{\limsup} \frac{1}{t} \log \Vert \veps_t \Vert \leq -(\frac{\sigma^2}{2} - L)
    $ almost surely for any $\veps_0\in\sR^n$. In particular, if $\sigma^2 > 2L$, the solution $\veps_t=\vzero$ is almost surely exponentially stable.
\end{corollary}
\section{Experimental Results}
In this section we show the effectiveness of our Neural SDE framework in terms of generalization, non-adversarial robustness and adversarial robustness. We use the SDE model architecture illustrated in Figure \ref{fig:model_architecture} during the experiment.
Throughout our experiments, we set $\vf(\cdot)$ to be a neural network with several convolution blocks. As to $\mG(\cdot)$ we have the following choices:
\begin{itemize}[noitemsep,leftmargin=*]
    \item \textbf{Neural ODE}, this can be done by dropping the diffusion term $\mG(\vh_t,t;\vv)=\mzero$.
    \item \textbf{Additive noise}, when the diffusion term is independent of $\vh_t$, here we simply set it to be diagonal  $\mG(\vh_t,t;\vv)=\sigma_t\mI$.
    \item \textbf{Multiplicative noise}, when the diffusion term is proportional to $\vh_t$, or $\mG(\vh_t,t;\vw)=\sigma_t\vh_t$.
    \item \textbf{Dropout noise}, when the diffusion term is proportional to the drift term $\vf(\vh_t,t;\vw)$, i.e. $\mG(\vh_t,t;\vv)=\sigma_t\text{diag}\{\vf(\vh_t,t;\vw)\}$.
\end{itemize}
Note the last three are our proposed Neural SDE with different types of randomness as explained in Section 3.1.
\subsection{Generalization Performance}
In the first experiment, we show small noise helps generalization. However, note that our noise injection is different from randomness layer in the discrete case, for instance, dropout layer adds Bernoulli noise at training time, but the layer are then fixed at testing time; whereas our Neural SDE model keeps randomness at testing time and takes the average of multiple forward propagation.
\par
As for datasets, we choose CIFAR-10, STL-10 and Tiny-ImageNet\footnote{Downloaded from \url{https://tiny-imagenet.herokuapp.com/}} to include various sizes and number of classes. The experimental results are shown in Table~\ref{tab:generalization-result}. We see that for all datasets, Neural SDE  consistently outperforms ODE, and the reason is that adding moderate noise to the models at training time can act as a regularizer and thus improves testing accuracy. Based upon that, if we further keep testing time noise and ensemble the outputs, we will obtain even better results. 

\begin{table}[htb]
\centering
\caption{\label{tab:generalization-result}Evaluating the model generalization under different choices of diffusion matrix $\mG(\vh_t,t;\vv)$ introduced above. For the last three noise types, we search a suitable parameter $\sigma_t$ for each so that the diffusion matrix $\mG$ properly regularizes the model. TTN means testing time noise. Model size is counted by \#parameters. }
\scalebox{0.75}{
\begin{threeparttable}
    \begin{tabular}{cccccccccc}
    \toprule
    \multirow{2}{*}{Data}  & \multirow{2}{*}{Model size}   &  \multicolumn{4}{c}{Accuracy@1 --- w/o TTN} &  \multicolumn{4}{c}{Accuracy@1 --- w/ TTN}\\
    \cmidrule(lr){3-6}  \cmidrule(lr){7-10}
   & & ODE & Additive & Multiplicative & Dropout & ODE & Additive & Multiplicative & Dropout\\
    \midrule
    CIFAR-10 & 115~K &  81.63 & 83.65 & 83.26 & 83.60 &  -- & 83.89 & 83.76 & \textbf{84.55}\\
    \midrule
    STL-10 & 2.44~M & 58.03 & 61.23 & 60.54 & 61.26 &  -- & 62.11 & \textbf{62.58} & 62.13\\
    \midrule
    Tiny-ImageNet & 2.49~M & 45.19 & 45.25 & 46.94 & 47.04 &  -- & 45.39 & 46.65 & \textbf{47.81}\\
    \bottomrule
    \end{tabular}
\end{threeparttable}
}
\end{table}
\subsection{Improved non-adversarial robustness}
In this experiment, we aim at evaluating the robustness of models under non-adversarial corruptions following the idea of \cite{hendrycks2019benchmarking}. The corrupted datasets contain tens of defects in photography including motion blur, Gaussian noise, fog etc.  For each noise type, we run Neural ODE and Neural SDE with dropout noise, and gather the testing accuracy. The final results are reported by mean accuracy (mAcc) in Table \ref{tab:non-adv-result} by changing the level of corruption. Both models are trained on completely clean data, which means the corrupted images are not visible to them during the training stage, nor could they augment the training set with the same types of corruptions. From the table, we can see that Neural SDE performs better than Neural ODE in 8 out of 10 cases.
For the rest two, both ODE and SDE are performing very close. 
This shows that our proposed Neural SDE can improve the robustness of Neural ODE under non-adversarial corrupted data.
\begin{table}[htb]
\centering
\caption{\label{tab:non-adv-result}Testing accuracy results under different levels of non-adversarial perturbations.}
\scalebox{0.8}{
\begin{threeparttable}
    \begin{tabular}{ccccccc}
    \toprule
    \multirow{2}{*}{Data} & \multirow{2}{*}{Noise type} & \multicolumn{5}{c}{mild corrupt $\leftarrow$ Accuracy $\rightarrow$ severe corrupt}\\
    \cmidrule(lr){3-7}
    & &  Level 1 & Level 2 & Level 3 & Level 4 & Level 5\\
    \midrule
    \multirow{3}{*}{CIFAR10-C\tnote{\textdagger}} & ODE & 75.89 & 70.59 & 66.52 & 60.91 & 53.02 \\
    & Dropout & 77.02 & 71.58 & 67.21 & 61.61 & 53.81 \\
    & Dropout+TTN & \textbf{79.07} & \textbf{73.98} & \textbf{69.74} & \textbf{64.19} & \textbf{55.99} \\
    \midrule
    \multirow{3}{*}{TinyImageNet-C\tnote{\textdagger}} & ODE & 23.01 & 19.18 & 15.20 & \textbf{12.20} & \textbf{9.88} \\
    & Dropout & 22.85 & 18.94 & 14.64 & 11.54 & 9.09 \\
    & Dropout+TTN & \textbf{23.84} & \textbf{19.89} & \textbf{15.28} & 12.08 & 9.44\\
    \bottomrule
    \end{tabular}
    \begin{tablenotes}
    \item[\textdagger] Downloaded from \url{https://github.com/hendrycks/robustness}
\end{tablenotes}
\end{threeparttable}}
\end{table}
\subsection{Improved adversarial robustness}
\begin{figure}[htb]
    \centering
    \includegraphics[width=0.8\linewidth]{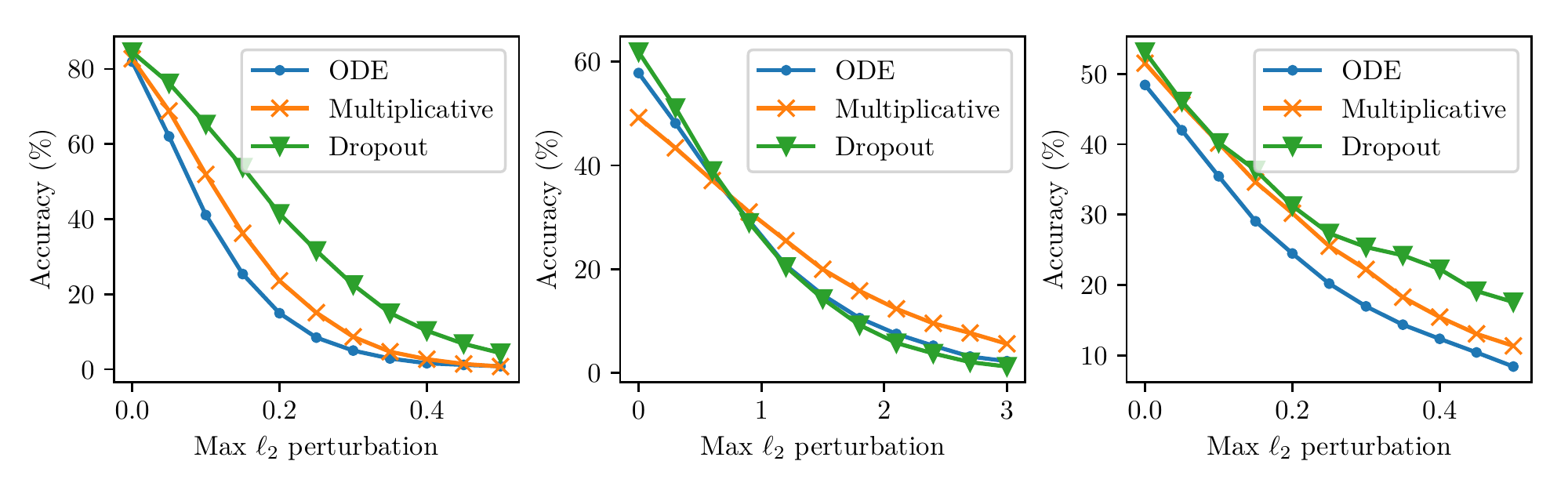}
    \caption{Comparing the robustness against $\ell_2$-norm constrained adversarial perturbations, on CIFAR-10 (left), STL-10 (middle) and Tiny-ImageNet (right) data. We evaluate testing accuracy with three models, namely Neural ODE, Neural SDE with multiplicative noise and dropout noise.}
    \label{fig:cmp-adv-acc}
\end{figure}
Next, we consider the performance of Neural SDE models under adversarial perturbation. Clearly, this scenario is strictly harder than previous cases: by design, the adversarial perturbations are guaranteed to be the worst case within a small neighborhood (ignoring the suboptimality of optimization algorithms) crafted through constrained loss maximization procedure, so it represents the worst case performance. In our experiment, we adopt multi-step $\ell_{\infty}$-PGD attack~\cite{madry2017towards}, although other strong white-box attacks such as C\&W~\cite{carlini2017towards} are also suitable. The experimental results are shown in Figure~\ref{fig:cmp-adv-acc}. As we can see both Neural SDE with multiplicative noise and dropout noise are more resistant to adversarial attack than Neural ODE, and dropout noise outperforms multiplicative noise.
\subsection{Visualizing the perturbations of hidden states}
\begin{wrapfigure}{r}{0.4\textwidth}
\begin{center}
    \includegraphics[width=0.39\textwidth]{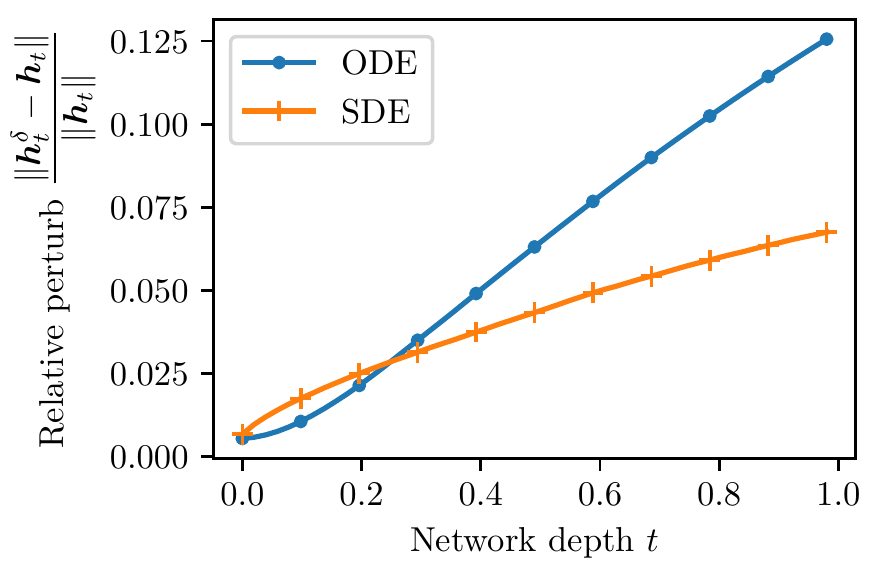}
\end{center}
\caption{\label{fig:visualize_perturb}Comparing the perturbations of hidden states, $\veps_t$, on both ODE and SDE (we choose dropout-style noise).}
\end{wrapfigure}
In this experiment, we take a look at the perturbation $\veps_t=\vh_t^e-\vh_t$ at any time $t$. Recall the 1-d toy example in Figure~\ref{fig:toy_example}, we can observe that the perturbation at time $t$ can be well suppressed by adding a strong diffusion term, which is also confirmed by theorem. However, it is still questionable whether the same phenomenon also exists in deep neural network since we cannot add very large noise to the network during training or testing time. 
If the noise is too large, it will also remove all useful features. Thus it becomes important to make sure that this will not happen to our models. To this end we first sample an input $\vx$ from CIFAR-10 and gather all the hidden states $\vh_t$ at time $t=[0, \Delta t, 2\Delta t,\dots, N\Delta t]$. Then we perform regular PGD attack~\cite{madry2017towards} and find the perturbation $\bm{\delta}_x$ such that $\vx_{\mathrm{adv}}=\vx+\bm{\delta}_x$ is an adversarial image, and feed the new data $\vx_{\mathrm{adv}}$ into network again so we get $\vh_t^e$ at the same time stamps as $\vh_t$. Finally we plot the error $\veps_t=\vh_t^e-\vh_t$ \textit{w.r.t.} time $t$ (also called ``network depth''), shown in Figure~\ref{fig:visualize_perturb}. We can observe that by adding a diffusion term (dropout-style noise), the error accumulates much slower than ordinary Neural ODE model.

\section{Conclusion}
To conclude, we introduce the Neural SDE model which can stabilize the prediction of Neural ODE by injecting stochastic noise. 
Our model can achieve better generalization and improve the robustness to both adversarial and non-adversarial noises. 

\section*{Acknowledgement}
We acknowledge the support by NSF IIS1719097, Intel, Google Cloud and AITRICS.

\bibliography{ref}
\bibliographystyle{unsrt}

\appendix

\newpage


\section{Proofs}

We present the proofs of theorems on stability of SDE. The proofs are adapted from \cite{mao2007stochastic}. We start with two crucial lemmas.

\begin{lemma}
\label{lem:ass-for-fG-Delta}
If $\vf, \mG$ satisfy Assumption (\ref{ass:lipschitz}), then $\vf_\Delta, \mG_\Delta$ satisfy Assumption (\ref{ass:sublinear},\ref{ass:lipschitz}).
\end{lemma}
\begin{proof}
By Assumption \eqref{ass:lipschitz} on $\vf, \mG$, we can obtain that for any $\veps, \tilde{\veps}\in\sR^n, t\geq 0$
\begin{align*}
    &\Vert \vf_\Delta(\veps,t)\Vert + \Vert \mG_\Delta(\veps,t)\Vert \leq c_2 \Vert \veps \Vert \leq c_2 (1+\Vert \veps \Vert),\\
    &\Vert \vf_\Delta(\veps,t) - \vf_\Delta(\tilde{\veps},t)\Vert + \Vert \mG_\Delta(\veps,t)- \mG_\Delta(\tilde{\veps},t)\Vert \leq c_2 \Vert \veps -\tilde{\veps}\Vert.
\end{align*}
This guarantees the uniqueness of the solution of \eqref{eq:sde}.
\end{proof}

\begin{lemma}
\label{lem:nonzero}
For \eqref{eq:sde}, whenever $\veps_0\neq \vzero$, $\Pr\{\veps_t\neq \vzero \text{ for all } t\geq 0\}=1$.
\end{lemma}
\begin{proof}
We prove it by contradiction. Let $\tau=\inf \{t\geq 0: \veps_t = 0\}$. Then if it is not true, there exists some $\veps_0\neq \vzero$ such that $\Pr \{\tau <\infty\}>0$. Therefore, we can find sufficiently large constant $T>0$ and $\theta>1$ such that
$\Pr(A) \coloneqq \Pr\{\tau <T \text{ and } \vert \veps_t \vert\leq \theta-1, \ \forall\ 0\leq t\leq \tau\}>0$.
By Assumption \ref{ass:lipschitz} on $\vf$ and $\mG$, there exists a positive constant $K_\theta$ such that 
\begin{equation}
\label{eq:lemma-B2-1}
    \Vert \vf_\Delta(\veps,t)\Vert + \Vert\mG_\Delta(\veps,t)\Vert \leq K_\theta \Vert \veps \Vert, \quad \text{ for all } \Vert \veps \Vert \leq \theta \text{ and }  0\leq t \leq T.
\end{equation}
Let $V(\veps, t)= \Vert\eps\Vert^{-1}$. Then, for any $0\leq \Vert \veps \Vert \leq \theta$ and $0\leq t\leq T$, we have
\begin{align}
    \mathcal{L}V(\veps, t) &= -\Vert \veps \Vert^{-3} \veps^\top \vf_\Delta(\veps, t) + \frac{1}{2}\{ -\Vert \veps \Vert^{-3} \Vert \mG_\Delta (\veps, t) \Vert^{2} + 3 \Vert \veps \Vert^{-5} \Vert \veps^\top \mG_\Delta(\veps, t) \Vert^2 \}\nonumber\\
    &\leq \Vert \veps \Vert^{-2}\Vert \vf_\Delta(\veps, t) \Vert + \Vert \veps \Vert^{-3} \Vert \mG_\Delta(\veps, t) \Vert^2\nonumber\\
    &\leq K_\theta \Vert \veps \Vert^{-1} + K_\theta^2 \Vert \veps \Vert^{-1} = K_\theta(1+K_\theta) V(\veps, t),
\end{align}
where the first inequality comes from Cauchy-Schwartz and the last one comes from \eqref{eq:lemma-B2-1}.
For any $\delta \in (0, \Vert \veps_0 \Vert)$, we define the stopping time $\tau_{\delta}\coloneqq \inf\{ t\geq 0:  \Vert \veps_t \Vert \notin (\delta, \theta)\}$.
Let $\nu_\delta = \min\{\tau_\delta, T\}$. By It\^{o}'s formula, $ \E\bigg[ e^{-K_\theta(1+K_\theta)\nu_\delta} V(\veps_{\nu_\delta}, \nu_\delta) \bigg]$
\begin{align}
\label{eq:lemma-B2-2}
    = V(\veps_0 ,0) + \E \int_{0}^{\nu_\delta}e^{-K_\theta(1+K_\theta)}s \bigg[ -K_\theta(1+K_\theta) V(\veps_s, s) + \mathcal{L}V(\veps_s,s)\bigg]\dif s \leq  \Vert \veps_0 \Vert^{-1}.
\end{align}
Since $\tau_\delta \leq T$ and $\Vert \veps_{\tau_\delta} \Vert=\delta$ for any $\omega\in A$, then \eqref{eq:lemma-B2-2} implies
\begin{equation}
    \E\bigg[ e^{-K_\theta(1+K_\theta)T} \delta^{-1} \mathbf{1}_{A} \bigg] =  \delta^{-1}e^{-K_\theta(1+K_\theta)T} \Pr(A) \leq  \Vert \veps_0 \Vert^{-1}.
\end{equation}
Thus, $\Pr(A) \leq \delta \Vert \veps_0 \Vert^{-1} e^{K_\theta(1+K_\theta)T}$. Letting $\delta\rightarrow0$, we obtain $\Pr(A)=0$, which leads to a contradiction.
\end{proof}

\subsection*{Proof of Theorem \ref{th:exp-stable}}
We then prove Theorem \ref{th:exp-stable}. Clearly, \eqref{eq:exp-stable} holds for $\veps_0 = \vzero$ since $\veps_t \equiv \vzero$. For any $\veps_0\neq \vzero$, we have $\veps_t \neq \vzero$ for all $t\geq 0$ almost surely by Lemma \ref{lem:nonzero}. Thus, by applying It{\^o}'s formula and condition (2), we can show that for $t\geq 0$,
\begin{equation}
\label{eq:thm3.2-1}
    \log V(\veps_t,t)\leq \log V(\veps_0, 0) + c_2t + M(t) - \frac{1}{2}\int_{0}^{t} \frac{\vert V_1(\veps_s, s) \mG_\Delta(\veps_s,s)\vert^2}{V^2(\veps_s,s)} \dif s.
\end{equation}
where $M(t) = \int_{0}^{t} \frac{V_1(\veps_s, s) \mG_\Delta(\veps_s,s)}{V(\veps_s, s)}\dif \rvB_s$ is a continuous martingale with initial value $M(0)=0$.
By the exponential martingale inequality, for any arbitrary $\alpha\in(0,1)$ and $n=1,2,\cdots$, we have
\begin{equation}
\label{eq:thm3.2-2}
    \Pr\bigg \{ \underset{0\leq t\leq n}{\sup} \bigg[ M(t) - \frac{\alpha}{2} \int_{0}^{t} \frac{\vert V_1(\veps_s, s) \mG_\Delta(\veps_s,s)\vert^2}{V^2(\veps_s,s)} \dif s \bigg] > \frac{2}{\alpha} \log n \bigg \} \leq \frac{1}{n^2}.
\end{equation}
Applying Borel-Cantelli lemma, we can get that for almost all $\omega\in \Omega$, there exists an integer $n_0=n_0(\omega)$ such that if $n\geq n_0$,
\begin{equation}
\label{eq:thm3.2-3}
    M(t) \leq \frac{2}{\alpha} \log n + \frac{\alpha}{2} \int_{0}^{t} \frac{\vert V_1(\veps_s, s) \mG_\Delta(\veps_s,s)\vert^2}{V^2(\veps_s,s)} \dif s,\quad \forall \ 0\leq t\leq n.
\end{equation}
Combining \eqref{eq:thm3.2-1}, \eqref{eq:thm3.2-3} and condition (3), we can obtain that
\begin{equation}
    \log V(\veps_t,t)\leq \log V(\veps_0, 0)  - \frac{1}{2} [(1-\alpha)c_3 - 2c_2] t + \frac{2}{\alpha} \log n.
\end{equation}
for all $0\leq t \leq n$ and $n\geq n_0$ almost surely. Therefore, for almost all $\omega\in \Omega$, if $n-1\leq t \leq n$ and $n\geq n_0$, we have
\begin{equation}
    \frac{1}{t}\log V(\veps_t,t) \leq -\frac{1}{2}[(1-\alpha)c_3 - 2c_2] + \frac{\log V(\veps_0, 0) + \frac{2}{\alpha} \log n}{n-1}
\end{equation}
which consequently implies
\begin{equation}
    \underset{t\rightarrow \infty}{\limsup} \frac{1}{t} \log V(\veps_t,t) \leq -\frac{1}{2}[(1-\alpha)c_3 -2c_2])\quad a.s.
\end{equation}
With condition (1) and arbitrary choice of $\alpha\in(0,1)$, we can obtain \eqref{eq:exp-stable}.

\subsection*{Proof of Corollary \ref{cor:stabilize}}
We apply Theorem \ref{th:exp-stable} to establish the theories on stability of \eqref{eq:multiplicative-error-SDE}. Note that $\vf(\vh_t, t; \vw)$ is $L$-Lipschitz continuous w.r.t $\vh_t$ and $\mG(\vh_t,t;\vv) = \sigma\vh_t, m=1$. Then, \eqref{eq:multiplicative-error-SDE} has a unique solution, with $\vf_\Delta$ and $\mG_\Delta$ satisfying Assumptions (\ref{ass:sublinear},\ref{ass:lipschitz})
\begin{align*}
    &\Vert \vf_\Delta(\veps_t,t)\Vert + \Vert \mG_\Delta(\veps_t,t)\Vert \leq \max \{L, \sigma\} \Vert \veps_t \Vert\leq \max \{L, \sigma\} (1+\Vert \veps_t \Vert),\\
    &\Vert \vf_\Delta(\veps_t,t) - \vf_\Delta(\tilde{\veps}_t,t)\Vert + \Vert \mG_\Delta(\veps_t,t)- \mG_\Delta(\tilde{\veps_t},t)\Vert \leq \max \{L, \sigma\} \Vert \veps_t -\tilde{\veps_t}\Vert.
\end{align*}
To apply Theorem \ref{th:exp-stable}, let $V(\veps,t) = \Vert \veps \Vert^2 $. Then,
\begin{align*}
    &\mathcal{L}V(\veps,t) = 2\veps^\top \vf_\Delta(\veps, t) + \sigma^2 \Vert\veps\Vert^2 \leq (2L+\sigma^2) \Vert \veps \Vert^2 =(2L+\sigma^2) V(\veps,t) ,\\
    &\Vert V_1(\veps, t) \mG_\Delta(\veps, t) \Vert^2 = 4\sigma^2 V(\veps,t)^2.
\end{align*}
Let $c_1 =1, p=2, c_2 = 2L+\sigma^2, c_3=4\sigma^2$. By Theorem \ref{th:exp-stable}, we finished the proof.
\end{document}